\documentclass[conference]{IEEEtran}

\usepackage{amsmath,amssymb,amsfonts}
\usepackage{mathtools}
\usepackage{bm}
\usepackage{array}
\usepackage{textcomp}
\usepackage{stfloats}
\usepackage{url}
\usepackage{graphicx}
\usepackage{booktabs}
\usepackage{enumitem}
\usepackage{cite}
\usepackage{algorithm}
\usepackage{algpseudocode}
\usepackage{amsthm}
\theoremstyle{plain}
\newtheorem{proposition}{Proposition}
\theoremstyle{definition}
\newtheorem{definition}{Definition}
\theoremstyle{remark}
\newtheorem*{remark}{Remark}

\newcommand{\R}{\mathbb{R}}
\newcommand{\Phimat}{\Phi}
\newcommand{\Ej}{E_j}
\newcommand{\tr}{\operatorname{tr}}
\newcommand{\T}{^{\!\top}}
\newcommand{\Wj}{W_j}
\newcommand{\ej}{e_j}

\begin{document}

\title{Trajectory-Aware Node Contributions and\\
the Limits of Static Controllability}

\author{\IEEEauthorblockN{Valentina V. Kuskova}
\IEEEauthorblockA{\textit{Lucy Family Institute for Data \& Society} \\
\textit{University of Notre Dame}\\
Notre Dame, IN, USA \\
vkuskova@nd.edu}
\and
\IEEEauthorblockN{Dmitry Zaytsev}
\IEEEauthorblockA{\textit{Lucy Family Institute for Data \& Society} \\
\textit{University of Notre Dame}\\
Notre Dame, IN, USA \\
zaytsevdi2@gmail.com}
\and
\IEEEauthorblockN{Michael Coppedge}
\IEEEauthorblockA{\textit{Department of Political Science} \\
\textit{University of Notre Dame}\\
Notre Dame, IN, USA \\
mcoppedg@nd.edu}
}

\maketitle

\begin{abstract}
\noindent
A recurring data-mining task in complex, time-evolving networks is to determine
how individual nodes contribute to system behavior. Existing approaches rely on
either static-graph centralities or control-theoretic quantities such as
controllability Gramians, which assume linear, time-invariant dynamics.
Estimated systems, however, are typically nonlinear and time-varying. We define
\emph{emergent contribution}, a finite-horizon measure of a node's dynamical
leverage: the metric-weighted energy of its impulse response accumulated along
the system trajectory. Computed from the Jacobians of any differentiable model,
the measure is estimator-agnostic and reduces \emph{exactly} to average
controllability in the linear, time-invariant limit.

Our contribution is therefore a characterization of \emph{when} the two
measures agree and when they diverge. Using a controlled synthetic family with
known ground-truth contribution, we construct a phase diagram spanning
nonlinearity, regime structure, persistence, and perturbation amplitude.
Emergent contribution and average controllability agree under static or smoothly
drifting dynamics and both track ground truth. Divergence emerges under
persistent regime switching, is strongest under persistent sign reversal, and
disappears when the sign reversal is removed. At extreme perturbation
amplitudes, both measures degrade, identifying the limits of local
linearization.

We then place five estimated real systems, spanning political, economic,
financial, and environmental domains, within this phase space. Their placement
serves as a diagnostic of when emergent contribution provides information beyond
static controllability and therefore justifies its additional computational
cost. On one panel examined in depth, a twenty-seed retraining ensemble reveals
a robust variance--leverage dissociation: nodes whose perturbations propagate
widely despite low within-system variance, and conversely high-variance nodes
whose perturbations do not propagate. This structure is recovered by neither
static centralities nor variance-based summaries.

\end{abstract}


\section{Introduction}
\label{sec:intro}

A recurring data-mining task in complex, time-evolving systems is to determine
which variables are dynamically consequential. Given a multivariate time series
and an estimated interaction structure, analysts often wish to identify the
components whose perturbations propagate most strongly through the learned
system. This problem arises across domains including network neuroscience,
systems biology, political economy, economics, and environmental science. In
each case, the objective is not merely to recover an interaction network, but
to mine from a fitted dynamical model which variables are most influential for
the system's future evolution.

One line of work approaches this problem through the topology of a static
network. Descriptive centralities such as degree, betweenness, closeness,
eigenvector centrality, and PageRank measure how a node is positioned within a
fixed pattern of connections. These measures remain widely used throughout
empirical network science. However, they characterize structural position
rather than dynamical influence. A central node is not necessarily one whose
perturbations substantially alter the system's trajectory.

A second line of work treats the network as a dynamical system and evaluates
influence through controllability. Structural
controllability~\cite{liu2011controllability}, the controllability Gramian, and
the average- and modal-controllability measures of Pasqualetti, Zampieri, and
Bullo~\cite{pasqualetti2014controllability} provide principled node-level
summaries of how perturbations propagate through a system. These methods
represent a substantial advance over static centralities because they quantify
dynamical rather than topological influence. Their interpretation, however,
depends on a key assumption: the system is linear and time-invariant,
$\dot{x}=Ax+Bu$.

The systems typically estimated from data are neither linear nor time
invariant. Their local dynamics depend on the state, and their governing
relationships may drift over time. Neural vector autoregressions, nonlinear
state-space models, neural ordinary differential equations, and related
estimators routinely operate in this regime. Yet controllability-based
quantities remain widely used as if the linear time-invariant approximation
were sufficient. This raises a fundamental data-mining question: \emph{when do
the node rankings implied by static controllability remain reliable, and when
do they fail?}

Answering that question requires a node-level measure that remains meaningful
for empirically estimated nonlinear, time-varying systems. We therefore define
\emph{emergent contribution} $\Ej$, a finite-horizon measure of a node's
dynamical leverage in an estimated dynamical system. The measure is constructed
from the state-transition matrix obtained by propagating the system's Jacobians
along its trajectory and requires only that the underlying estimator be
differentiable. It therefore applies to graphical and neural vector
autoregressions, neural ordinary differential equations, and state-space models
alike. The underlying mathematical object itself, finite-horizon
impulse-response energy along a Jacobian-product trajectory, is not new
(Section~\ref{sec:ire}). Our contribution is its formalization as an
interpretable, estimator-agnostic, and gauge-explicit node-level measure for
empirically estimated nonlinear, time-varying systems.

The framework contributes four elements needed to study node-level contribution
in learned dynamical systems. First, emergent contribution reduces exactly to
average controllability in the LTI limit (Section~\ref{sec:reduction}). Second,
it enables a controlled characterization of when that approximation fails
through a synthetic phase diagram (Section~\ref{sec:synthetic}). Third, it
resolves practical issues required for application to learned models, including
gauge choice, companion-form lag handling, horizon selection, and
estimator-agnostic computation. Fourth, we evaluate the framework across five
empirical systems spanning political, economic, financial, and environmental
domains (Section~\ref{sec:application}); the resulting placements within the
synthetic phase space serve as system-level diagnostics for when the measure
provides information beyond static controllability.
Throughout, we distinguish dynamical leverage from broader notions of
importance. Emergent contribution quantifies how strongly perturbations to a
component propagate through an estimated system. It does not imply that a
high-leverage component is inherently more valuable, fundamental, or important
in a normative sense. Such claims require additional domain-specific
arguments. Our focus here is solely on the definition, interpretation, and
behavior of the measure itself.

\section{Related work}
\label{sec:related}

\paragraph{Static network measures} A large literature quantifies node importance through graph-theoretic
centralities such as degree, strength, betweenness, closeness, eigenvector
centrality, and PageRank~\cite{freeman1978centrality, page1999pagerank}. These
measures summarize a node's position within a static network and have proved
highly useful across domains. In network
psychometrics~\cite{borsboom2013network, epskamp2018network}, for example, they
support the view of psychological constructs as emergent properties of systems
of interacting components. Their limitation for the present problem is
fundamental: they characterize topology rather than dynamics. Two systems with
identical adjacency structures but different dynamical laws have identical
centralities, even though a node may strongly influence one system while having
little effect in the other. Related critiques of cross-sectional network
analysis have shown that node-level summaries can conflate dynamical influence
with latent-variable structure~\cite{hallquist2021problems}, motivating a shift
from topological importance to dynamical influence.

\paragraph{Controllability-based measures} A second line of work evaluates node influence through the lens of dynamical
systems and control theory. Structural
controllability~\cite{liu2011controllability} studies whether a system can be
driven to desired states through inputs at selected nodes, while the
controllability Gramian quantifies the energy required to do so. Per-node
summaries such as average controllability and modal
controllability~\cite{pasqualetti2014controllability} provide principled
measures of dynamical influence and have become widely used in network
neuroscience~\cite{gu2015controllability}. These measures move beyond topology
by explicitly accounting for system dynamics. Their interpretation, however,
rests on the assumption of a linear time-invariant system. When dynamics are
nonlinear or vary over time, the fixed matrix $A$ underlying Gramian-based
analysis no longer exists, and the standard controllability framework no longer
applies directly.

\paragraph{Node influence in learned dynamical systems}
The challenge is particularly acute for modern data-mining workflows, where
dynamical systems are increasingly estimated using flexible nonlinear models.
Neural ordinary differential equations~\cite{chen2018neural}, neural
Granger-causal vector autoregressions~\cite{tank2021neural,
bussmann2021neural}, and related architectures routinely produce
state-dependent, time-varying dynamics. Yet the tools used to interpret node
importance in such models often remain either graph-theoretic centralities or
controllability measures applied under a linear approximation. As a result,
there is no widely used operational measure of node-level contribution
that is simultaneously finite-horizon, applicable to empirically estimated
nonlinear systems, and directly computable from a fitted model.

Our approach addresses this gap by defining a trajectory-based measure of
node-level dynamical leverage that depends only on the Jacobians of the
estimated system. The measure reduces exactly to average controllability in the
linear time-invariant limit while remaining applicable to nonlinear and
time-varying dynamics. Because it is formulated entirely in terms of the
estimated Jacobians, it is agnostic to the choice of underlying dynamical
estimator and can be applied wherever differentiable system dynamics are
available.

\paragraph{Relation to nonlinear and operator-theoretic control}
Controllability of nonlinear systems is a mature area of control theory,
typically studied through Lie-algebraic reachability conditions or empirical and
finite-horizon controllability Gramians for nonlinear
systems~\cite{lall2002empirical}. Operator-theoretic approaches, most notably
the Koopman framework, represent nonlinear dynamics through a linear evolution
in a lifted observable space and are often paired with data-driven
approximations such as dynamic mode
decomposition~\cite{williams2015datadriven}. Our objective
is narrower and more operational. Rather than performing a full nonlinear
controllability analysis, we seek a finite-horizon, \emph{per-node} scalar that
can be computed directly from a fitted model's Jacobians and that reduces
exactly to the average-controllability Gramian trace in the linear limit. While
empirical Gramians integrate responses to explicit inputs and Koopman methods
seek a global linearizing representation, emergent contribution propagates local
Jacobians along the realized trajectory and accumulates the resulting
impulse-response energy node by node. The result is a lightweight,
estimator-agnostic summary of dynamical leverage rather than a reachability
certificate.

\paragraph{Relation to sensitivity and attribution methods}
Emergent contribution is related to trajectory-sensitivity
analysis~\cite{khalil2002nonlinear} and to gradient-based attribution methods
such as saliency, integrated gradients~\cite{sundararajan2017axiomatic}, and
neural Granger-causal attribution~\cite{tank2021neural}. Where integrated
gradients accumulate gradients along a path in \emph{input} space to attribute
a single prediction, emergent contribution propagates perturbations through
the full state-transition operator along a path in \emph{time}, accumulating
influence over a declared horizon. Its exact reduction to average
controllability in the LTI limit additionally provides a control-theoretic
interpretation unavailable to purely attribution-based measures.

\paragraph{Influence and spreading on networks}
A separate literature studies a node's capacity to spread influence through a
network, including independent-cascade and linear-threshold models and the
influence-maximization problem they motivate~\cite{kempe2003maximizing}. Like
emergent contribution, these approaches quantify propagation. The difference is
that they operate on a prescribed stochastic diffusion process over a fixed
graph, whereas emergent contribution propagates the \emph{estimated}
deterministic dynamics of a fitted multivariate system and therefore inherits
its nonlinearity and time variation. The two perspectives are complementary:
one studies spreading behavior under an assumed contagion process, while the
other measures dynamical leverage in a learned dynamical system.

We are particularly motivated by settings in which system-level behavior is
viewed as emerging from interactions among constituent components, as in
network psychometrics~\cite{borsboom2013network, epskamp2018network}. Our
contribution is not a new static centrality, but a dynamical measure of
node-level contribution for nonlinear and time-varying systems. In this sense, the
framework extends existing approaches to node importance by providing an
operational notion of contribution in precisely the settings where classical
controllability theory no longer applies directly.

\section{Setup and notation}
\label{sec:setup}

Consider a discrete-time dynamical system on $n$ components. After collapsing any
finite lag structure into companion form, write the one-step map as
\begin{equation}
  x_{t+1} = f_t(x_t) + \varepsilon_t, \qquad x_t \in \R^n .
\end{equation}
The map $f_t$ may be \textbf{nonlinear} and may depend on $t$
(\textbf{time-varying}); $\varepsilon_t$ is mean-zero noise. The local causal
influence of the components on one another is the Jacobian of the map evaluated
along the realized trajectory,
\begin{equation}
  A_t(x) = \left.\frac{\partial f_t}{\partial x}\right|_{x} \in \R^{n\times n},
\end{equation}
whose $(i,j)$ entry is the marginal effect of component $j$ at time $t$ on
component $i$ at time $t+1$, given the current state $x$.

Multi-step propagation is governed by the \textbf{state-transition matrix}, the
ordered product of Jacobians along the trajectory:
\begin{equation}
\Phimat_t^{\tau}(x)
= A_{t+\tau-1}(x_{t+\tau-1}) \cdots A_{t+1}(x_{t+1})\, A_t(x_t),
\end{equation}
with $\Phimat_t^{0}(x) = I$ and the trajectory generated by $x_t = x$ and
$x_{s+1} = f_s(x_s)$. The $j$-th column, $\Phimat_t^{\tau}(x)\,\ej$, is the
\textbf{local impulse response}: the location of a small perturbation to
component $j$ at time $t$ after $\tau$ steps of propagation. Because the product is evaluated along the realized
trajectory, it captures both nonlinearity (through state-dependent Jacobians)
and time variation (through changes in the Jacobians across time). The
state-transition matrix is the fundamental object underlying all subsequent
definitions. The distinction between simulated and observed trajectories becomes important
because the state-transition matrix depends on the sequence of Jacobians along
which it is evaluated.

\paragraph{Choice of trajectory}
The definition above generates the trajectory by iterating the estimated map,
$x_{s+1} = f_s(x_s)$, yielding a \textbf{model-simulated} trajectory. This
choice isolates the intrinsic dynamics implied by the fitted model. In
empirical applications, one may instead evaluate Jacobians along the
\textbf{observed realized} trajectory $\{x_s\}$ from the data, addressing a
different question: how leverage propagated along the path the system actually
followed rather than along the path predicted by the model. The two coincide
under perfect model fit and otherwise differ. We use the simulated trajectory
throughout the formal development and theoretical reductions below, and note
explicitly whenever an empirical diagnostic substitutes the observed
trajectory.

\section{The construct: emergent contribution}
\label{sec:construct}

\begin{definition}[Emergent contribution]
The emergent contribution of component $j$, at time $t$, from state $x$, over
horizon $T$ is
\begin{equation}
  \Ej(t,x;T)
  = \sum_{\tau=0}^{T-1}
    \bigl(\Phimat_t^{\tau}(x)\,\ej\bigr)\T M
    \bigl(\Phimat_t^{\tau}(x)\,\ej\bigr),
  \label{eq:def}
\end{equation}
with $\Phimat_t^{0}(x) = I$ and $M$ a fixed positive-definite metric that
declares what ``system impact'' means. The default is $M = I$ on standardized
components (see Section~\ref{sec:gauge}).
\end{definition}

\paragraph{Interpretation}
$\Ej$ is the total $M$-weighted energy of component $j$'s impulse response,
accumulated over a horizon of length $T$. Components whose perturbations
propagate broadly or persist over time have high emergent contribution, whereas
components whose perturbations remain localized or decay rapidly have low
contribution.

\paragraph{Naming}
Outside the LTI limit, $\Ej$ is not a controllability measure in the
reachability sense of control theory. Rather, it is a finite-horizon measure of
dynamical leverage, or equivalently, impulse-response energy accumulated along a
trajectory. The quantity can be viewed as a trajectory-level generalization of
the average-controllability Gramian trace, with state-dependent Jacobian
products $\Phimat_t^{\tau}(x)$ replacing matrix powers $A^{\tau}$. In the LTI
limit, the two coincide exactly (Section~\ref{sec:reduction}). Outside that
regime, we interpret $\Ej$ as finite-horizon dynamical leverage rather than a
reachability certificate. Accordingly, we reserve the term `average
controllability'' for the LTI reduction and use `dynamical leverage'' or
``impulse-response energy'' elsewhere.

\paragraph{Finite-horizon by design}
$\Ej$ is defined for a finite horizon $T$, and $T$ is a declared parameter rather
than a limit to be taken. This matters because systems with spectral radius near
or above $1$ can exhibit growing impulse-response energy, causing $\Ej$ to
diverge as $T \to \infty$. The finite-horizon definition is therefore a modeling
choice rather than a technical convenience: $\Ej$ measures leverage
\emph{accumulated over a stated horizon}, and any reported value must specify
$T$. Comparisons across nodes, states, or times must hold $T$ fixed.

\paragraph{The self-term ($\tau=0$) is retained}
The $\tau=0$ term contributes $\ej\T M \ej$ to $\Ej$ and should not be removed.
For $M = I$, it equals $1$ for every node, but subtracting a constant from every
node does \emph{not} preserve the normalized weights of
Section~\ref{sec:weights}, because normalization is not shift-invariant.
Accordingly, the self-term is retained; one should not ``exclude the diagonal.''

\subsection{From leverage to weights (an optional aggregation)}
\label{sec:weights}

If one wishes to aggregate components into a single score -  a step requiring a
separate, domain-specific normative argument that we do not make here - the
contribution induces aggregation weights
\begin{equation}
  w_j(t,x) = \frac{\Ej(t,x;T)}{\sum_k E_k(t,x;T)} .
\end{equation}
These weights are \textbf{state-dependent and time-dependent by construction}.
This dependence is intentional: a component's contribution to system dynamics
depends on the regime the system occupies and on how the dynamics have evolved.
Throughout, we use these weights only as a normalized reporting device for
relative leverage, not as a normative aggregation.

\subsection{What $\Ej$ measures, and what it does not}
\label{sec:leverage}

$\Ej$ measures the \textbf{dynamical leverage} of component $j$ in the
\emph{estimated} system: the extent to which perturbations to $j$ propagate
through the system's dynamics. By itself, it does \textbf{not} establish that
$j$ is more important, more valuable, or more fundamental than other
components. High leverage is compatible with at least three distinct
interpretations:
\begin{enumerate}[label=(\alph*), nosep, leftmargin=2.2em]
\item a \textbf{load-bearing} component on which system behavior depends;
\item a \textbf{fragile amplifier} that transmits shocks destructively; or
\item a \textbf{highly endogenous relay} that primarily propagates the
behavior of other components.
\end{enumerate}

The statement that ``component $j$ has dynamical leverage $\Ej$ in the
estimated system'' is descriptive and carries no normative implication. Any
claim that leverage should determine importance - for example, that an
aggregate score ought to weight components by leverage - requires a separate,
domain-specific justification that lies outside the scope of this paper.
Throughout, we interpret $\Ej$ as a measure of dynamical leverage rather than a
normative measure of importance.

\subsection{Impulse-response energy and the contribution}
\label{sec:ire}

Mathematically, $\Ej$ is finite-horizon impulse-response energy accumulated
along a trajectory. The novelty therefore lies not in a new dynamical quantity,
but in its formalization as an operational node-level measure for empirically
estimated systems. The remainder of the paper establishes its exact reduction
to average controllability in the linear time-invariant limit
(Section~\ref{sec:reduction}), addresses the practical issues required for
computation on learned dynamical models
(Sections~\ref{sec:gauge}--\ref{sec:estimator}), and characterizes the
conditions under which it departs from static controllability
(Sections~\ref{sec:synthetic}--\ref{sec:application}).

\subsection{Algorithm}
\label{sec:algorithm}

Algorithm~\ref{alg:ej} gives the operational computation of emergent
contribution. A node-perturbation vector is propagated through the Jacobian
sequence via repeated matrix--vector products, accumulating metric-weighted
energy at each step, including the $\tau=0$ self-term.

\begin{algorithm}[t]
\caption{Computation of emergent contribution}
\label{alg:ej}
\begin{algorithmic}[1]
\Require fitted differentiable model $f$; trajectory $x_t,\dots,x_{t+T-1}$;
horizon $T$; metric $M$; lag order $K$
\Ensure emergent contributions $\{\Ej\}_{j=1}^{n}$ and weights $\{w_j\}$
\If{$K > 1$} construct companion coordinates; state dimension $d \gets nK$
\Else{} $d \gets n$
\EndIf
\For{$s = 0,\dots,T-2$}
  \State $A_s \gets$ Jacobian of $f$ at $x_{t+s}$ \Comment{autodiff; companion if $K>1$}
\EndFor
\For{$j = 1,\dots,n$}
  \State $v \gets e_j^{(0)}$ \Comment{perturb node $j$ in current-state block}
  \State $E_j \gets v^{\!\top} M v$ \Comment{$\tau=0$ self-term, retained}
  \For{$\tau = 1,\dots,T-1$}
    \State $v \gets A_{\tau-1}\, v$ \Comment{propagate: $v = \Phi^{\tau} e_j^{(0)}$}
    \State $E_j \gets E_j + v^{\!\top} M v$
  \EndFor
\EndFor
\State $w_j \gets E_j / \sum_k E_k$ \Comment{normalized leverage weights}
\State \Return $\{E_j\}, \{w_j\}$
\end{algorithmic}
\end{algorithm}

\subsection{Computational complexity}
\label{sec:complexity}

Let $d$ denote the state dimension ($d=n$ for $K=1$ and $d=nK$ in companion
coordinates) and $T$ the horizon. Jacobian extraction requires $T-1$ autodiff
evaluations, with cost determined by the model architecture; the additional
companion-coordinate blocks come from the same per-lag Jacobian and require no
separate evaluations.

Propagation is specific to emergent contribution: each node requires $O(Td^2)$
matrix--vector operations, yielding $O(nTd^2) = O(n^3K^2T)$ overall in
companion coordinates. This is cheaper than explicitly forming the matrix
products $\Phimat^\tau$, which costs $O(Td^3)$. Memory usage is $O(d^2)$ for
the current Jacobian plus $O(d)$ per propagated vector. Once Jacobians are
available, the remaining cost is matrix--vector products and is typically
inexpensive relative to model fitting.

\section{Scale dependence and the gauge choice}
\label{sec:gauge}

Scale dependence is a fundamental property of $\Ej$, not a secondary
implementation detail. Because $\Ej$ is defined as a squared $M$-norm of impulse
responses, both its value and the induced weights depend on the units of the
components, their standardization, and the choice of coordinates. Numerically,
rescaling components by a diagonal matrix $D$ while holding $M=I$ fixed can shift a single node's weight share across most of its possible range, from a small minority to a large majority of the total. Consequently, any
interpretation of $\Ej$ requires an explicit choice of coordinate system and
metric.

\paragraph{The gauge}
Our default choice is to compute $\Ej$ on \textbf{standardized components}
(each component centered and scaled to unit variance) with $M=I$, and to define
``system impact'' in standardized-component units. This is a modeling choice and
must be stated explicitly. It is also the natural default, because it places all
components on a common scale before measuring how perturbations propagate
through the system.

\paragraph{Coordinate transformation law}
Under a linear change of coordinates $x' = D x$, the dynamics transform as
$A' = D A D^{-1}$ and the impulse responses as
$\Phimat' = D\,\Phimat\,D^{-1}$. The contribution remains invariant if and only
if the metric transforms according to
$M' = D^{-\top} M D^{-1}$. Declaring the standardized metric $M=I$ in
standardized coordinates therefore fixes the gauge. Any reported value of
$\Ej$ is understood relative to that declared metric.

\section{Main result: reduction to average controllability in the LTI limit}
\label{sec:reduction}

The key theoretical property of emergent contribution is that it reduces
\emph{exactly} to the average controllability measure of Pasqualetti, Zampieri,
and Bullo~\cite{pasqualetti2014controllability} in the linear time-invariant
(LTI) limit. The measure therefore extends a well-established quantity rather
than introducing an unrelated notion of node influence.

\begin{proposition}[LTI-limit reduction]
\label{prop:reduction}
Suppose $f_t(x) = A x$ for a constant matrix $A$ (the LTI limit), take $M = I$,
and use Definition~\eqref{eq:def} (sum $\tau = 0,\dots,T-1$, $\Phimat^0 = I$).
Then
\begin{equation}
  \Ej(t,x;T) = \mathrm{AC}_j(T) = \tr \Wj(T),
\end{equation}
independent of $t$ and $x$, where
\begin{equation}
  \Wj(T) = \sum_{\tau=0}^{T-1} A^{\tau}\,\ej \ej\T\,(A\T)^{\tau}
\end{equation}
is the node-$j$ average-controllability Gramian.
\end{proposition}

\begin{proof}
In the LTI limit the Jacobian is constant, $A_t(x) = A$, so the state-transition
matrix is a matrix power, $\Phimat_t^{\tau}(x) = A^{\tau}$, with $A^0 = I$. Hence,
taking $M = I$,
\begin{align}
  \Ej(t,x;T)
  &= \sum_{\tau=0}^{T-1} \bigl\| A^{\tau} \ej \bigr\|^2 \notag\\
  &= \sum_{\tau=0}^{T-1} \ej\T (A\T)^{\tau} A^{\tau} \ej \notag\\
  &= \sum_{\tau=0}^{T-1} \tr\!\bigl( A^{\tau} \ej \ej\T (A\T)^{\tau} \bigr)
     && \text{(cyclic property of trace)} \notag\\
  &= \tr \sum_{\tau=0}^{T-1} A^{\tau} \ej \ej\T (A\T)^{\tau} \notag\\
  &= \tr \Wj(T) = \mathrm{AC}_j(T). && \qedhere
\end{align}
\end{proof}

\begin{remark}[Indexing convention]
The reduction requires the sum in Definition~\eqref{eq:def} to run over
$\tau=0,\dots,T-1$ with the $\Phimat^0=I$ term included. Using instead
$\tau=1,\dots,T$ yields
$\tr \Wj(T+1)-(\ej\T M \ej)$, which both adds the $A^T$ term and removes the
$A^0=I$ term. This difference is not merely a uniform shift: because the
weights of Section~\ref{sec:weights} are obtained by normalization, changing the
indexing alters the resulting weights.
\end{remark}

\paragraph{Continuous-time version}
The same reduction holds in continuous time. In the LTI limit,
$\Phimat(\tau)=e^{A\tau}$, giving
\begin{equation}
  \Ej = \int_0^{T} \bigl\| e^{A\tau} \ej \bigr\|^2 \, d\tau
      = \tr \int_0^{T} e^{A\tau} \ej \ej\T e^{A\T \tau} \, d\tau,
\end{equation}
which is precisely the trace of the finite-horizon continuous-time
controllability Gramian. The discrete-time expression is exactly the
finite-horizon discrete Gramian trace; its correspondence with the
continuous-time form arises through the usual discretization limit rather than
through direct equality at finite step size.

\section{When emergent contribution differs from average controllability}
\label{sec:novelty}

The exact reduction of Section~\ref{sec:reduction} establishes that emergent
contribution recovers average controllability whenever the assumptions of the
linear time-invariant model hold. The central question is therefore not whether
$\Ej$ differs from average controllability, but under what conditions the two
begin to diverge.

\paragraph{Nonlinearity}
When $f_t$ is nonlinear, the state-transition matrix
$\Phimat_t^{\tau}(x)$ is formed from \emph{state-dependent} Jacobians rather
than powers of a fixed matrix. Consequently, $\Ej$ becomes state-dependent: a
node's dynamical leverage may vary across operating regimes even when the
underlying network structure remains unchanged.

Nonlinearity alone is not the primary source of divergence from
average controllability. Across the synthetic systems studied in
Section~\ref{sec:phase}, purely nonlinear but otherwise stable dynamics produce
rankings that remain close to those implied by average controllability. The
substantial departures arise instead when nonlinearity interacts with
persistent regime structure, causing the realized Jacobian products to differ
systematically from powers of their average. Under those conditions, the rank
correlation between emergent contribution and average controllability declines
substantially and node orderings are reorganized. The generalization is
therefore non-trivial, but for a specific and identifiable reason rather than
because nonlinearity alone is present.

\paragraph{When $\Ej$ departs from average controllability}
The practical question is when emergent contribution changes the answer rather
than reproducing the familiar controllability ranking. The decomposition
provides the underlying mechanism. Average controllability replaces the realized
sequence of Jacobian products with powers of an average Jacobian,
$\bar A^{\tau}$. The two measures therefore agree when the realized products
remain close to powers of $\bar A$, and diverge when they do not.

Sign reversal provides a sufficient mechanism for departure that is
transparent enough to analyze directly. Suppose a
pathway propagates perturbations positively in one regime and negatively in
another. Averaging the Jacobians causes these opposing effects to cancel, making
the corresponding entry of $\bar A$ appear weak. Emergent contribution, by
contrast, accumulates propagation along the realized trajectory and therefore
retains the contribution associated with each regime. The result is a discrepancy
between the node's realized contribution and its controllability ranking
under the averaged system.

Sign reversal is not the only mechanism that can produce divergence, nor do we
claim that it explains any particular empirical system. Instead,
Section~\ref{sec:phase} characterizes the phenomenon against known ground truth
using a controlled sweep of nonlinearity, regime structure, persistence, and
perturbation amplitude. Across these experiments, departures from average
controllability are negligible for static and smoothly drifting dynamics,
emerge under persistent regime change, and become strongest when persistent
sign reversal interacts with nonlinearity. At sufficiently large perturbation
amplitudes, however, the local-linearization approximation underlying $\Ej$
itself begins to degrade.

The resulting conclusion is deliberately bounded. Emergent contribution is most
informative when the estimated dynamics exhibit persistent regime-dependent
variation while remaining within an operating range for which local
linearization remains faithful.

\paragraph{Time-variation}
When $f_t$ varies over time, the state-transition matrix combines Jacobians
drawn from different dynamical regimes. As a result, $\Ej$ becomes explicitly
time-dependent, tracking how node-level dynamical leverage changes as the
estimated system evolves.

\paragraph{Scope of the novelty}
The novelty of emergent contribution does not lie in introducing nonlinear or
time-varying controllability theory, both of which are well established within
control theory. Rather, it lies in providing an \textbf{operational,
finite-horizon measure of node-level contribution for empirically
estimated nonlinear, time-varying systems}. This is the setting encountered in
many applications of network neuroscience and empirical network science, where
average controllability is typically computed under the linear time-invariant
approximation. Emergent contribution fills this operational gap: it is
computable from estimated Jacobians, reduces exactly to average controllability
in the LTI limit, and remains meaningful when the underlying dynamics are
nonlinear, time-varying, or regime-dependent.

\section{Estimator-agnosticism and uncertainty}
\label{sec:estimator}

\paragraph{Estimator-agnostic}
$\Ej$ requires only the Jacobians $A_t(x)$ evaluated along a trajectory. Any
differentiable dynamical model supplies them, including graphical or neural
VARs, neural ODEs, and state-space models. In the application below, the
Jacobians are obtained from a neural Granger-causal vector-autoregressive model,
but the definition of $\Ej$ does not depend on that choice. Changing the
estimator changes the Jacobians supplied to the measure, not the measure
itself. The contribution is therefore the \emph{measure}: a procedure for
extracting node-level dynamical leverage from a fitted dynamical
model, with the estimator serving as a replaceable source of Jacobians.

\paragraph{The companion operator}
For a model with $K$ lags, the Jacobian used for $\Ej$ (and for any stability
analysis) should be the $(nK \times nK)$ \textbf{companion} Jacobian, whose
eigenvalues are the true multipliers of the $K$-lag system. The leading
lag-$1$ block alone is not the correct multi-step propagation operator because
it omits the contribution of longer lags to system dynamics. Using the
companion form ensures that the state-transition product
$\Phimat_t^{\tau}$ correctly represents propagation when $K > 1$.

\paragraph{What ``perturbing node $j$'' means in companion coordinates.}
In companion form, the lifted state stacks the current and lagged components,
$z_t = (x_t, x_{t-1}, \dots, x_{t-K+1}) \in \R^{nK}$. A perturbation to node
$j$ is applied to the current-state block only. The perturbation vector
$e_j^{(0)} \in \R^{nK}$ equals $1$ in the $j$-th position of the leading
($x_t$) block and $0$ elsewhere, including all lagged copies of $j$.
Perturbing every lagged copy simultaneously corresponds to a different input
and does not represent a perturbation to node $j$ at time $t$. Accordingly,
in companion coordinates, $\Ej$ uses $e_j^{(0)}$ in place of $\ej$
throughout Definition~\eqref{eq:def}. The state-transition operator
$\Phimat_t^{\tau}$ then propagates that current-state perturbation into the
lagged blocks at later times according to the system dynamics, making the node
definition unambiguous for $K > 1$.

\paragraph{Uncertainty propagation}
Because $\Ej$ is a functional of estimated Jacobians, it inherits uncertainty
from the underlying dynamical model. In empirical applications, uncertainty
should therefore be quantified through an explicit resampling or ensemble
procedure rather than reported solely as a point estimate. The choice of
resampling scheme determines which sources of uncertainty are captured.
\textbf{Resampling units} (here, countries as the natural cluster in a country
panel) capture sampling uncertainty over the population of units;
\textbf{block bootstrapping in time} captures uncertainty arising from temporal
dependence; \textbf{resampling residuals} capture innovation uncertainty
conditional on the fitted dynamics; and \textbf{re-running the estimator with
different random seeds} captures optimization stochasticity. Each procedure
induces a distribution over the contributions $\Ej$ and the corresponding
weights $w_j$. For the panel studied here, a natural default is a country
(cluster) bootstrap, optionally combined with a seed ensemble to distinguish
sampling variability from estimation variability. The resulting uncertainty is a
computed quantity whose interpretation depends on the chosen resampling scheme.

\paragraph{Time aggregation}
When a single per-unit score is desired, one must specify where the
state-dependent quantity $\Ej$ is evaluated: at the current state, averaged
along a trajectory, or at a basin equilibrium. Each choice yields a different
summary of dynamical leverage and may be appropriate in different applications.
Accordingly, the aggregation scheme should be stated explicitly and justified in
context.

\section{Synthetic validation}
\label{sec:synthetic}

We validate emergent contribution in two stages: a simple planted system that
illustrates the departure from average controllability, followed by a controlled
phase diagram that characterizes when that departure occurs across a broader
space of dynamics.

\paragraph{A warm-up}
We begin with a six-node saturating system,
$x_{t+1} = \tanh(W x_t)$. Node~$0$ is a hub with strong outgoing connections to
nodes $1,2,3$, whereas node~$4$ contains a saturation-immune self-loop. Near
the origin, $A(x)\approx W$ and $\Ej$ coincides with the
average-controllability Gramian trace to machine precision. As the hub's
pathways saturate, however, its leverage collapses and shifts to node~$4$
(hub share $0.46\to 0.15$, node~$4$ share $0.14\to 0.23$), producing a
reordering that average controllability, which depends only on
$W$, cannot capture.

\paragraph{Companion-operator ablation}
For $K>1$, the companion operator rather than the lag-1 Jacobian block is the
correct propagation operator. On a controlled $\mathrm{VAR}(2)$ system with
non-trivial lag-2 structure, the lag-1 surrogate not only understates
contributions but also \emph{selects the wrong top node}
(rank correlation $0.80$ relative to the correct companion ordering). The
choice therefore affects substantive conclusions rather than merely numerical
magnitudes.

\paragraph{Sign-cancellation: an isolated departure}
We next isolate one mechanism by which $\Ej$ can depart from average
controllability. The system alternates between two regimes, with occasional
transitions between them. One node~$d$ acts as a persistent driver, but its
downstream pathway reverses sign across regimes while all other edges remain
fixed. Within each regime, $d$ is strongly contributive; in the averaged
Jacobian, however, the positive and negative pathway contributions cancel,
causing average controllability to view the node as weak.

Using a ground-truth contribution computed by propagating perturbations through
the true nonlinear dynamics, $\Ej$ recovers the node ordering exactly and ranks
$d$ first, matching the ground truth. Average controllability, by contrast,
misranks $d$ toward the middle because it cannot recover the regime-specific
sign structure. A corresponding control experiment removes the sign reversal
while leaving the remainder of the construction unchanged. In that setting, the
discrepancy disappears entirely. Sign cancellation is therefore a sufficient
condition for separation in this controlled example, though not necessarily the
only one. The broader space of departures is characterized by the phase diagram
that follows.

\subsection{A phase diagram of fidelity to ground truth}
\label{sec:phase}

The experiments above establish particular instances of departure. We now
characterize \emph{when} emergent contribution tracks true contribution, when
average controllability suffices, and when both fail. To do so, we generate
trajectories from a two-regime nonlinear system,
\[
  x_{t+1} = A_{r_t}\,x_t + \alpha\,\tanh(B_{r_t} x_t) + \varepsilon_t,
\]
and systematically vary four factors: \emph{nonlinearity}
($\alpha \in \{\text{low},\text{med},\text{high}\}$); \emph{regime structure}
(static, smooth drift, persistent reroute, persistent sign reversal);
\emph{regime persistence} relative to the evaluation horizon; and
\emph{perturbation amplitude} used for the ground-truth probe.

For each cell, we compute three quantities across 20 random seeds:
$\rho(\Ej,E_{\text{true}})$,
$\rho(\mathrm{AC},E_{\text{true}})$, and
$\rho(\Ej,\mathrm{AC})$, where $E_{\text{true}}$ is obtained by finite-amplitude
perturb-and-propagate through the \emph{true} nonlinear dynamics rather than
through an estimated model. Average controllability is evaluated under two
baselines: a single mean-state linearization (\emph{pooled}) and the average of
the realized Jacobians (\emph{realized}).

The primary outcome is fidelity to ground truth rather than the magnitude of
the $\Ej$--AC discrepancy itself. A departure is only meaningful if emergent
contribution is the measure that more accurately tracks the true contribution.

\begin{table}[!t]
\centering
\caption{Ground-Truth Fidelity Advantage of $\Ej$ over Pooled Average
Controllability, by Regime Structure and Nonlinearity.}
\label{tab:phase}
\begin{tabular}{lccc}
\toprule
Regime structure & low & med & high \\
\midrule
Static            & $0.01$ & $0.05$ & $0.03$ \\
Smooth drift      & $0.00$ & $-0.01$ & $0.01$ \\
Persistent reroute & $0.01$ & $0.04$ & $0.10$ \\
Sign reversal     & $0.26$ & $0.47$ & $0.58$ \\
\bottomrule
\end{tabular}
\end{table}

\begin{figure*}[!t]
\centering
\includegraphics[width=0.92\textwidth]{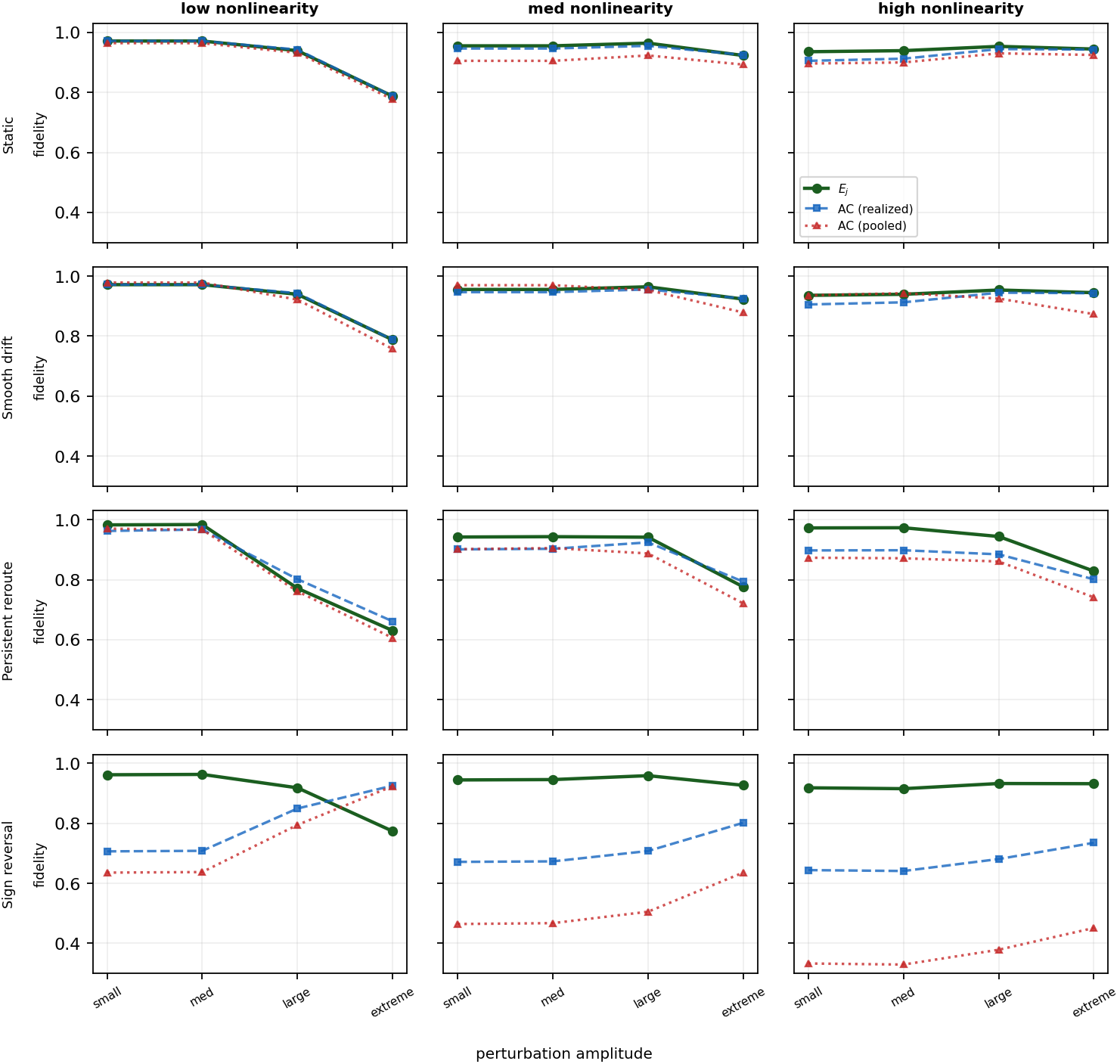}
\caption{\textbf{Phase diagram: when emergent contribution departs from average
controllability.} Each panel reports fidelity to ground-truth contribution
(Spearman rank correlation) as perturbation amplitude varies for emergent
contribution $\Ej$ (solid green), average controllability computed from the
average realized Jacobian (dashed blue), and average controllability from a
single mean-state linearization (dotted red). Rows correspond to regime
structure (top to bottom: static, smooth drift, persistent reroute, persistent
regime-dependent sign reversal); columns correspond to nonlinearity (left to
right: low, medium, high). Under static and smoothly drifting dynamics, all
methods closely track ground truth. Under persistent reroute, $\Ej$ exhibits a
modest fidelity advantage. Under persistent sign reversal, $\Ej$ substantially
outperforms both controllability baselines, while the pooled linearization
degrades sharply. Shaded regions mark the extreme-amplitude boundary, beyond
which the local-linearization approximation underlying $\Ej$ begins to lose
fidelity and all methods converge.}
\label{fig:phase}
\end{figure*}

\paragraph{Four regimes}
Define
$\Delta=\rho(\Ej,E_{\text{true}})-\rho(\mathrm{AC}_{\text{pooled}},E_{\text{true}})$
as the fidelity advantage of emergent contribution over a single mean-state
linearization, with positive values indicating that $\Ej$ more closely tracks
ground truth. Table~\ref{tab:phase} and Figure~\ref{fig:phase} reveal four
distinct operating regimes. The pattern is particularly clear in the
seed-level robustness analysis: across 20 seeds, the advantage of $\Ej$ is
significant in $97\%$ of sign-reversal cells and $56\%$ of persistent-reroute
cells, but only $8$--$17\%$ of static and smooth-drift cells.

\emph{(i) Controllability-sufficient.} Under static or smoothly drifting
dynamics, $\Ej$, both controllability baselines, and ground truth nearly
coincide ($\Delta\approx0$) across all levels of nonlinearity. Smooth drift is
an especially informative negative control, showing that time variation alone
does not induce separation.

\emph{(ii) $\Ej$ separates.} Under persistent sign reversal at moderate and high
nonlinearity, $\Ej$ maintains fidelity near $0.93$, whereas the pooled baseline
falls to approximately $0.35$--$0.48$. The resulting advantage reaches
$\Delta=0.47$ (95\% CI $[0.32,0.63]$) at medium nonlinearity and
$\Delta=0.58$ ($[0.40,0.77]$) at high nonlinearity, with both intervals
excluding zero.

\emph{(iii) Intermediate.} Persistent reroute - a change in the dominant driver
without sign reversal - produces a smaller but consistent advantage
($\Delta\le0.10$), significant in more than half of the corresponding cells.
This indicates that persistent regime change in general can generate
separation, with sign reversal representing the strongest case rather than the
only one.

\emph{(iv) Both fail.} At extreme perturbation amplitudes, the local
linearization underlying $\Ej$ begins to lose fidelity. Under sign reversal,
fidelity declines from approximately $0.94$ to $0.88$ as amplitude increases
and converges toward the controllability baselines. This defines the validity
boundary of the measure: a local-linearization approach cannot faithfully track
sufficiently large-amplitude propagation. Consistent with this interpretation,
under low nonlinearity and extreme amplitude the pooled baseline can outperform
$\Ej$ ($0.92$ versus $0.77$). When exploitable nonlinear structure is weak and
the perturbation lies outside the faithful regime, $\Ej$ has little to gain and
its approximation has more to lose.

Together, these four regimes define a practical decision rule, summarized in
Table~\ref{tab:summary}.
\begin{table}[!t]
\centering
\caption{Main Empirical Findings: When to Use $\Ej$.}
\label{tab:summary}
\begin{tabular}{lll}
\toprule
Regime structure & $\Ej$ vs.\ AC & Practitioner takeaway \\
\midrule
Static / smooth drift & agree & AC suffices \\
Persistent reroute    & $\Ej$ ahead & worth checking \\
Sign reversal         & $\Ej$ wins & use $\Ej$ \\
Extreme amplitude     & both fail & beyond either \\
\bottomrule
\end{tabular}
\end{table}

\paragraph{A hierarchy of approximations}
Figure~\ref{fig:phase} supports a secondary but important conclusion: the
realized-Jacobian baseline (dashed) consistently outperforms the single
mean-state linearization (dotted). Averaging realized Jacobians already
recovers much of the trajectory structure discarded by a single linearization.
The additional gain provided by $\Ej$ comes from preserving the \emph{order} of
propagation rather than merely averaging it. The implication is not that static
controllability is uninformative, but that there exists a hierarchy of
approximations: single linearization, averaged Jacobians, and ordered
trajectory products, with $\Ej$ occupying the most faithful end of that
hierarchy in the regimes identified by the phase diagram.

\paragraph{Exact reduction across horizons}
Across random constant matrices $A$, system sizes
$n \in \{3,4,5,6\}$, and horizons
$T \in \{5,8,12,20\}$, the discrepancy
$\max_j |\Ej-\tr W_j(T)|$ is zero to machine precision, confirming
Proposition~\ref{prop:reduction} numerically.

\section{Placing real systems in the phase diagram}
\label{sec:application}

The synthetic experiments characterize the measure itself; we now ask where
estimated real systems fall within that characterization. Using the same neural
vector-autoregressive model (lag three, companion Jacobians, horizon $T=8$), we
analyze five public multivariate panels spanning distinct domains and measure
how far the $\Ej$ ordering departs from average controllability under the two
baselines of Section~\ref{sec:phase}.

The democracy panel comprises 16 component indicators from the Varieties of
Democracy (V-Dem) project~\cite{vdem15dataset},
capturing electoral, deliberative, participatory, liberal, and egalitarian
dimensions of democracy for 89 countries over 1950--2024. Following a
pre-registered temporal split at year 2000, leverage is estimated using the
pre-cutoff window (4{,}094 country-years after lag construction), with
post-2000 observations held out. The development panel comprises 8 World Bank
World Development Indicators~\cite{worldbank_wdi}; the macro-finance panel
contains 9 macroeconomic series from the Global Macro
Database~\cite{gmd2025}; the realized-volatility panel contains daily realized
volatilities for 8 major global equity indices~\cite{son2023forecasting}; and
the air-quality panel contains 11 pollutant and meteorological variables from
multi-site monitoring in Beijing, 2013--2017~\cite{zhang2017beijing}. These
panels differ substantially in dimensionality, sampling frequency, and
underlying substrate by design.

All panels are modeled using the same neural vector-autoregressive
architecture, optimization procedure, and training schedule. Full model
specification, preprocessing details, train--validation splits, and the frozen
synthetic data-generating process are \textbf{ released in the anonymized reproducibility
bundle}.\footnote{For code and configuration while the paper is under review, please contact the authors. }

\begin{table}[!t]
\centering
\caption{Five Estimated Real Systems Placed in the Phase Diagram.
Realized departures are reported as seed-ensemble means $\pm$ s.d., as the
real-data gaps are small, seed-sensitive quantities (see text).}
\label{tab:domains}
\begin{tabular}{lccl}
\toprule
Domain & $n$ & realized & region \\
\midrule
Macro-finance          & 9  & $0.09 \pm 0.04$ & intermediate \\
Democracy              & 16 & $0.04 \pm 0.03$ & moderate \\
Economic development   & 8  & $0.02 \pm 0.01$ & near-zero \\
Air quality            & 11 & $0.01 \pm 0.01$ & near-zero \\
Realized volatility    & 8  & $0.01 \pm 0.01$ & near-zero \\
\bottomrule
\end{tabular}
\\[2pt]
{\footnotesize\raggedright
Realized departure is $1-\rho(\Ej,\mathrm{AC}_{\text{realized}})$, reported as
the mean $\pm$ s.d. over a retraining ensemble (8 seeds; 3 for air quality, 20
for democracy). The gaps are small and seed-sensitive, so we report
distributions rather than point estimates. The pooled-baseline comparison is
even more variable across seeds and is discussed in the text rather than
tabulated.\par}
\end{table}

\paragraph{Reading the placement against the stronger baseline}
In Table~\ref{tab:domains}, departure is defined as
$1-\rho(\Ej,\mathrm{AC})$. We focus on departure from the
\emph{realized}-Jacobian baseline because Section~\ref{sec:phase} identified it
as the stronger controllability approximation: averaging realized Jacobians
already recovers much of the trajectory structure discarded by a single
linearization.

Viewed through this baseline, the five domains form a gradient consistent with
the synthetic characterization. Economic development ($0.02$), air quality
($0.01$), and realized volatility ($0.01$) exhibit essentially no departure,
indicating that $\Ej$ provides little additional information beyond the realized
Jacobian average. This near-zero placement is robust across the retraining
ensemble: every seed places all three of these panels at or below $0.03$. These
systems therefore occupy the controllability-sufficient region of the phase
diagram, consistent with dynamics that are effectively static or smoothly
varying over the evaluation horizon.

Macro-finance ($0.09$) and democracy ($0.04$) exhibit larger departures,
suggesting richer regime structure. The magnitude of the spread under the
realized baseline is modest ($0.01$--$0.09$), but this is itself informative:
the measure remains close to controllability where the phase diagram predicts it
should and departs only on systems displaying stronger dynamical heterogeneity.
The separation between these two regime-changing panels and the three near-zero
physical panels holds across every seed of the retraining ensemble, rather than
arising from a single fit.
A diagnostic of realized regime structure finds highly persistent but
sign-stable drivers in macro-finance and partial sign reversal in democracy,
corresponding to two of the departure mechanisms identified in
Section~\ref{sec:phase} (persistent reroute and sign reversal). 

The pooled-baseline comparison, departure from a single mean-state
linearization, is the more permissive of the two, but it is markedly less stable
across retrainings on every panel: pooled departures vary widely (for example,
roughly $0.02$--$0.33$ on economic development), so we do not report pooled point
estimates and instead lead on the realized baseline throughout. This instability
is itself consistent with the hierarchy of approximations from
Section~\ref{sec:phase}: a single linearization is a weak summary of trajectory
structure, so its agreement with $\Ej$ is both smaller and more fragile than that
of the averaged-Jacobian baseline, which remains strongly and stably aligned with
$\Ej$ ($\rho\approx0.93$--$0.98$ across seeds on the democracy panel).

Taken together, these results support a practical decision rule. Emergent
contribution provides the greatest value on systems that depart from the
controllability-sufficient corner of the phase diagram, and a system's realized
departure offers an empirical diagnostic of where it lies.

\paragraph{One domain in depth - a robust variance--leverage dissociation}
We examine the V-Dem democracy panel as a worked example using a twenty-seed
retraining ensemble (validation MSE $0.043 \pm 0.0003$). Dynamical leverage does
not merely restate how much a component moves. Across components, $\Ej$ is
essentially uncorrelated with within-country temporal variance (Spearman
$0.03$). By contrast, $\Ej$ correlates substantially with
one-step influence (out-strength $0.80$) and two-step reach ($0.78$).

The dissociation is visible in specific components. \emph{Elected officials}
has the highest within-country variance of all sixteen indicators yet
consistently appears in the bottom leverage tier (typically rank 14 across
seeds). Likewise, \emph{suffrage} combines the second-highest variance with
bottom-tier leverage (rank 16 in the modal seed). Both variables move
substantially, but their movement does not propagate widely through the
estimated system.

These conclusions are robust to estimation noise. Across the twenty retrained
models, the highest- and lowest-leverage components remain sharply identified:
freedom of association and freedom of expression occupy the top two ranks in
every seed (each spanning ranks 1--3), while suffrage and direct vote occupy
the bottom two (ranks 14--16). Individual liberty and elected officials remain
consistently within the top and bottom tiers, respectively, although their
exact ranks vary. Uncertainty is concentrated in the middle of the ordering,
where the remaining ten components reshuffle freely (mean pairwise Spearman
$0.85$). The practical implication is that leverage tiers are reliably
identified, whereas fine distinctions among middle-ranked components are not.

\section{Discussion and Conclusion}
\label{sec:discussion}

We introduced a measure of finite-horizon dynamical leverage for estimated
nonlinear, time-varying systems, showed that it reduces exactly to average
controllability in the linear time-invariant limit, and characterized the
conditions under which the two measures diverge. Across synthetic systems with
known ground truth and across five empirical domains, emergent contribution
captures structure that neither static controllability nor variance-based
summaries recover.

\paragraph{What the measure is and is not} Emergent contribution measures
dynamical leverage, not importance: the claim is descriptive, that a component
has a particular leverage within the estimated system, and any normative
interpretation of that leverage requires a separate domain-specific argument.

\paragraph{Limitations and extensions} The measure is finite-horizon by design. For systems near or beyond marginal
stability, impulse-response energy grows with the evaluation horizon, making
the choice of $T$ an explicit modeling decision that must be stated and held
fixed when comparing results. As a functional of estimated Jacobians, $\Ej$
also inherits estimation uncertainty; the retraining ensemble used for the
democracy panel illustrates one approach to quantifying that uncertainty.
Finally, several natural extensions remain open, including modal-controllability
analogues and definitions of contribution relative to a basin of attraction in
multi-equilibrium systems.

\paragraph{Conclusion} We defined a finite-horizon measure of node-level contribution for estimated
nonlinear, time-varying systems, proved its exact reduction to average
controllability in the linear time-invariant limit, and characterized the
conditions under which the two measures diverge. Departures emerge under persistent regime change, with sign reversal the
strongest case. In those settings, emergent contribution provides a practical,
estimator-agnostic measure that recovers the trusted linear answer when
appropriate and extends it when the dynamics demand more.

\bibliographystyle{IEEEtran}
\bibliography{paper1}

\end{document}